\documentclass[fleqn,10pt,twocolumn]{SWARM25}
\usepackage{hyperref}
\usepackage{cite}
\usepackage{amsmath}
\usepackage{amssymb}    
\usepackage{booktabs}
\usepackage{graphicx}
\usepackage{capt-of}
\usepackage{subcaption}
\usepackage{changepage}
\usepackage{caption}
\usepackage{float}
\usepackage{multirow}

\usepackage{soul} 

\newcommand{\E}{\mathcal{E}}
\newcommand{\V}{\mathcal{V}}
\usepackage{algorithm}
\usepackage{algpseudocode}
\sethlcolor{yellow}
\renewcommand{\arraystretch}{1.5} 

\title{VariAntNet: Learning Decentralized Control of Multi-Agent Systems}

\author{Yigal Koifman\thanks{These authors contributed equally to this work.}${}^{\dagger*}$,
        Erez Koifman\footnotemark[1]$ {}^{*}$\footnotemark[1],
        Eran Iceland ${}^{}$, 
        Ariel Barel ${}^{}$, and
        Alfred M. Bruckstein ${}^{}$
        }

 \speaker{Yigal Koifman}

\affils{${}^{}$Department of Computer Science, Technion Israel Institute of Technology, Haifa, Israel\\
(Tel: +972-73-378-4957; E-mail: igal.k@cs.technion.ac.il)\\
}

\abstract{%
A simple multi-agent system can be effectively utilized in disaster response applications, such as firefighting. Such a swarm is required to operate in complex environments with limited local sensing and no reliable inter-agent communication or centralized control.
These simple robotic agents, also known as Ant Robots, are defined as anonymous agents that possess limited sensing capabilities, lack a shared coordinate system, and do not communicate explicitly with one another.
A key challenge for simple swarms lies in maintaining cohesion and avoiding fragmentation despite limited-range sensing. Recent advances in machine learning offer effective solutions to some of the classical decentralized control challenges.
We propose VariAntNet, a deep learning-based decentralized control model designed to facilitate agent swarming and collaborative task execution. VariAntNet includes a preprocessing stage that extracts geometric features from unordered, variable-sized local observations. 
It incorporates a neural network architecture trained with a novel, differentiable, multi-objective,  mathematically justified loss function that promotes swarm cohesiveness by utilizing the properties of the visibility graph Laplacian matrix.
VariAntNet is demonstrated on the fundamental multi-agent gathering task, where agents with bearing-only and limited-range sensing must gather at some location. VariAntNet significantly outperforms an existing analytical solution, achieving more than double the convergence rate while maintaining high swarm connectivity across varying swarm sizes. 
While the analytical solution guarantees cohesion, it is often too slow in practice and fails to meet the required convergence time. In time-critical scenarios, such as emergency response operations where lives are at risk, rapid convergence is crucial, making slower analytical methods impractical and justifying the loss of some agents within the swarm. This paper presents and analyzes this trade-off in detail.
}


\begin{document}

\pagestyle{headings}

\maketitle

\begingroup
\renewcommand\thefootnote{}\footnotetext{${}^{*}$ These authors contributed equally to this work.}
\addtocounter{footnote}{-1}
\endgroup

\begin{figure}[!t]
\begin{center}
\hspace*{-4.5mm}
\includegraphics[width=1.07\columnwidth]{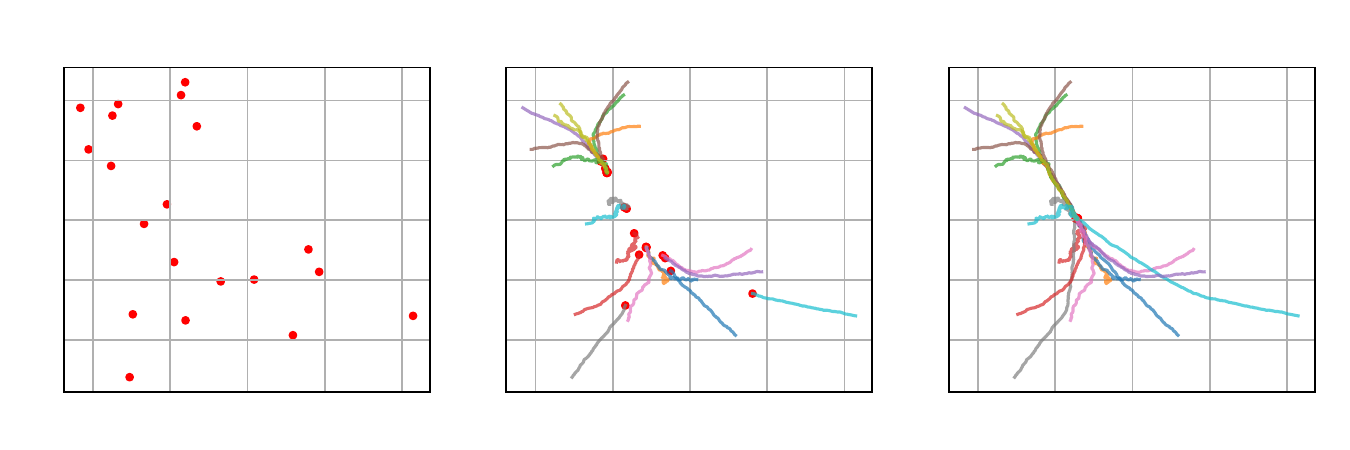}
\caption{An example of 20-agent swarm convergence using the VariAntNet model on a random constellation at steps 0, 75, and 200.}
\label{fig: An example of swarm convergence}
\end{center}
\end{figure}

\section{Introduction}
Swarm robotics, inspired by biological phenomena, has been extensively researched for disaster response, surveillance, and exploration applications.
Swarms are highly effective in complex tasks, particularly in harsh environmental conditions such as firefighting and natural disaster response ~\cite{roldan2021survey}. These environments pose significant challenges, such as unreliable communication and a lack of GPS or compass readings.
As shown in \cite{boanradio}, the intense heat in large fires ionizes the air, turning it into plasma, which degrades communication.
Therefore, decentralized and autonomous control is crucial for cooperation on collective tasks.

We present VariAntNet, a novel Deep Learning-based Decentralized Control (DLDC) that implements Centralized Training with Decentralized Execution (CTDE).
To our knowledge, unlike previous approaches, VariAntNet offers a novel combination of deep learning, geometric processing, and visibility graph Laplacian-based connectivity loss function for decentralized swarm control without communication.
VariAntNet is adaptive to unordered, varying-sized local observations while aiming to keep swarm cohesion.

To ensure the swarm is cost-effective and highly robust, we utilize Ant Robots, i.e., simple, inexpensive agents with limited sensing capabilities, no memory, no self-localization, and no inter-swarm communication. The agents sense their local environment and move accordingly. Lacking memory, an agent is incapable of either identifying or tracking its neighbors and must therefore rely solely on real-time observations.

Decentralized control of simple agent swarm has been extensively studied (see Section ~\ref{sec: Related Work}), primarily relying on analytical geometric algorithms.
In contrast, VariAntNet introduces an innovative DLDC framework that implements CTDE. 
During the training phase, the Neural Network (NN) optimizes a centralized loss function calculated based on the geometric constellation of the entire swarm. In decentralized inference, each agent poses a copy of the trained control network,  processes its local observation, and moves accordingly.

Certain geometric properties of agents' observations can significantly impact the design and performance of NN-based solutions. VariAntNet introduces a novel architecture that effectively addresses the following geometric properties:
\begin{itemize}
    \item \textbf{Unordered Observations:} Agents sense their environment and process the results in random order, creating multiple representations for identical observations. This results in different outputs for different representations of the same observation. VariAntNet produces the same result, regardless of the order in which neighboring agents are sensed.

    \item \textbf{Variable Observation Size:}
    Swarms contain varying numbers of agents. Furthermore, during convergence, each agent discovers an increasing number of neighbors, leading to observations of varying sizes. Traditional NNs are typically designed to address the maximal input length, using padding in partial inputs. This solution is inefficient in terms of time, number of parameters, and limits the input size. VariAntNet is designed to process effectively varying numbers of observed neighbors. 
    \item \textbf{Lack of Shared Coordinate System:} Agents operate without a common frame of reference. Each Agent detects its neighbors relative to its local coordinate system, which changes over time. As a result, geometrically identical constellations are represented differently depending on the agent's rotation.
\end{itemize}

The performance of VariAntNet is evaluated on the fundamental multi-agent gathering problem, where agents with bearing-only sensing and limited-range visibility need to converge at some location while maintaining cohesion, as shown in Fig.\ref{fig: An example of swarm convergence}.
The results of VariAntNet are compared to an analytical solution, with a focus on convergence rate and cohesiveness preservation.

In this study, we compare a deterministic, geometry-based method \cite{bellaiche2017continuous}, which guarantees convergence of all agents under all initial conditions, with a novel statistical approach based on machine learning, introduced here for the first time. Unlike the deterministic method, our statistical approach inherently carries a non-zero probability that some agents may fail to converge in certain scenarios. 

Despite this potential drawback, we investigate whether the statistical approach can offer a meaningful advantage in convergence rate. The deterministic method, while fully fault-tolerant, follows a conservative dynamic that limits acceleration, especially in the final phases of convergence, when agents are already in close proximity. Due to the limitations of bearing-only sensing, agents lack distance measurements, making it difficult (if not impossible) for the deterministic method to safely accelerate during late-stage convergence.

In contrast, the statistical method may generalize across a diverse range of geometric configurations encountered during training, potentially enabling more audacious motion strategies. We direct the reader to Fig.\ref{fig: Model Comparison}, which illustrates this hypothesis and summarizes the results of our empirical evaluation. Notably, the figure highlights a ``golden point'' 
in the design space: a 30-agent swarm initialized with a visibility ratio of $VR=0.875$, as explained in Section \ref{sec: Results}.

To summarize, our contribution is twofold: we propose a statistical approach that enables fast convergence, and an innovative, efficient, and compact architecture, tailored to the unique requirements of swarm systems, invariant to the size, order, and rotation of observations. The architecture is both adaptable and extensible, making it well-suited for studying additional swarm-based tasks and behaviors. Moreover, the innovative integration of the visibility graph Laplacian matrix provides a way to assess swarm connectivity in many other CTDE-based network designs.

\section{Related Work}
\label{sec: Related Work}

Decentralized control of swarms has been a longstanding research topic. Until recent years, it has primarily been studied using analytical methods that incorporate geometric considerations, gradient descent analysis, Lyapunov functions, and more. One of the classic problems studied in this field is the Gathering problem, where a group of autonomous agents is required to converge to a point or a bounded region without communicating for coordination while maintaining swarm connectivity to prevent swarm segmentation. The basic assumption is that agents are identical and equipped with limited-range, omni-directional sensors.
In a comprehensive survey \cite{barel2019come}, several analytical algorithms for solving the Gathering problem under different assumptions were introduced \cite{ando1999distributed}, \cite{gazi2003stability}, among many others.

Learning-based control methods show impressive success in solving swarm robotics problems. Some of these are described in surveys such as ~\cite{blais2023reinforcement},~\cite{nguyen2020deep}, and ~\cite{gronauer2022multi}.
Some of the learning methods focus on utilizing communication to improve agents’ learning capabilities in the environment, specifically on learnable communication protocols as described in~\cite{zhu2024survey}. These methods are distinct in their approaches and the various problem settings, restrictions, and different swarm tasks.
An actor-critic Reinforcement Learning (RL) approach was proposed in \cite{huttenrauch2017guided}, where a global observer critic employs deep Q-learning to enable limited-range sensing agents to avoid collisions and locate targets. \cite{huttenrauch2019deep} also suggested an RL-based approach for the rendezvous and evasion problem of limited-distance sensing agents, comparing different approaches for observation encoding, showing that NN-based encoding performs best in a bounded environment with partial localization and required inter-agent communication.

 A path planning algorithm with obstacle avoidance was introduced in \cite{islam2019path} for UAVs that need to collectively cover changing mission areas using a location-based Q-values method, based on a few basic discrete actions.

Another approach described in \cite{omidshafiei2017deep} relates to decentralized Multi-Agent Reinforcement Learning (MARL), and deals with generalizing across multiple tasks without explicit task identification in partial observability and a limited communication setup. The problem is defined within the Decentralized Partially Observable Markov Decision Processes (Dec-POMDPs) framework, and was solved by Decentralized Hysteretic Deep Recurrent Q-Networks (Dec-HDRQNs) and Concurrent Experience Replay Trajectories (CERTs) to stabilize decentralized learning and later distillization of the policies into a single, general multi-task policy.

VariAntNet introduces a simple and effective single-phase rotationally equivariant preprocessing, as depicted in Fig.~\ref {fig: rotational equivariance}, with an explicit geometric loss function, enabling robust and inherently scalable swarm size with no memory or tracking historical data and trajectories. 

One of the challenging tasks involves finite visibility and bearing-only sensing, where agents can measure relative direction to other agents but cannot estimate their distances. This specific case was studied by Bellaiche et al. ~\cite{bellaiche2017continuous}, who proposed a local motion rule that provably gathers agents in finite time. To achieve swarm convergence, each agent identifies the smallest sector of its visibility disk that bounds all its neighbors. If this sector spans an angle of less than $\pi$, the agent sets its velocity vector as the sum of the two unit vectors pointing toward its external neighbors. Otherwise, it remains stationary.

Our research presents a NN-based approach that provides a solution to the gathering problem under the constraints described in Section~\ref{sec: problem setting}. A comparison of the analytical solution and VariAntNet is presented and analyzed in Section~\ref{sec: experiments and results}.

\section{Problem Setting} \label{sec: problem setting}

As a testbed for VariAntNet, we tackled the multi-agent gathering problem of a swarm of agents where the desired behavior causes the swarm to gather to a small, bounded location while maintaining cohesion. The need for such algorithms arises in scenarios where agents operate in dangerous zones without communication and are equipped with bearing-only, limited-range sensors such as basic cameras.

\subsection {Assumptions}
\label{sec:Assumption }
We assume that agents are identical and indistinguishable, lack explicit communication capabilities, do not share a common geometric frame of reference, are oblivious, and have a bearing-only limited sensory range. The initial constellation of the swarm is assumed to be connected, and if an agent becomes disconnected from the swarm, it may not be able to rejoin, making separation undesirable.
\subsection {Notation}
\label{sec: Notation}
We denote agents' position at time $t$ as a vector \({P}(t)\triangleq\{{p}_i(t)\}_{i=1, 2, \ldots,	N} \text{, where }{p}_i(t)\triangleq\{(x_i,y_i)^T\}_{i=1, 2, \ldots,N},\)
and the distance between $p_i \text{ and } p_j$ as the Euclidean distance is defined as :
$$d(p_i, p_j) = \parallel{p_j} - {p_i} \parallel \triangleq{{\left((p_j-{p_i})^T ({p_j} -{p_i})\right)}}^{\frac{1}{2}},$$
We define $V$ as the agent's visibility range, meaning that agents $p_i, p_j$ mutually sense each other if \(d(p_i, p_j)\leq V\). 
Graphs are used to represent relationships between agents in multi-agent systems.
A swarm is represented by a visibility graph \( G(\mathcal{V}, {\E}) \), where:
\(\mathcal{V}\) = \(\{{\nu}_1, \nu_2, \ldots, \nu_n \} \) is the set of vertices representing the agents, and
\(\E \) is the set of edges:
$${\E}=\left\{\{\nu_i, \nu_j\}: \nu_i,\nu_j\in{\V} \space \land \space d(p_i,p_j)\leq V \right\}.$$
A swarm is said to be \textit{cohesive} if and only if the graph \( G(\V, \E) \) is a \textit{connected graph}, i.e., there exists a path between any two vertices \( \nu_i \) and \( \nu_j \) in \( G \). 

The agents sensed by agent $i$ at time $t$ are defined as the set of agent \(i\)'s neighbors, \( {\mathcal{N}}_i(t)\), in the visibility graph \( G( \V, \E) \):\\
$${\mathcal{N}}_i(t)= \left\{\nu_j \in  {\V}: \{\nu_i,\nu_j\} \in {\E} \land i \neq j \right\}.$$

Since the agents have bearing-only sensing capabilities, as depicted in Fig.~\ref{fig: neighbors notation}, agent \(i\)'s observation of agent \(j\) at time \(t\) is defined the unit vector, \(\hat{u}_{ij}\), pointing from \(p_i(t)\) to \(p_j(t)\):
$$\hat{u}_{ij}(t)=\frac{p_j(t) - p_i(t)}{\parallel p_j(t) - p_i(t) \parallel}$$
\begin{figure}[t!]
\begin{center}
\includegraphics[width=3cm]{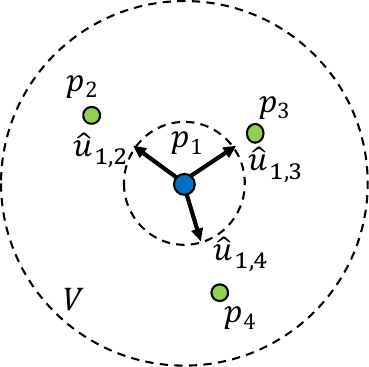}
\caption{Agent $p_1$ observation is given as unit vectors pointing toward its neighbors, constrained by bearing-only sensing with limited range \( V\).}
\label{fig: neighbors notation}
\end{center}
\end{figure}

Each agent's observation denoted by $O_i(t)$, is then defined as multi-set of unit vectors pointing at other agents that are visible to it:
$$O_i(t)=\{\hat{u}_{ij}(t) : \forall \nu_j \in {\mathcal{N}}_{i}(t)\}$$

The observation \(O_i(t)\) is also addressed as a matrix where each column is an observation:
$$O_i(t)=\begin{bmatrix}
 \hat{u}_{i1}(t),\hat{u}_{i2}(t),& \dots &,\hat{u}_{ij}(t) \\
\end{bmatrix}.$$

We define the observation space \(O\) as the set of all possible observations, which are all finite multi-sets of unit vectors.

Finally, we define an action as a directional unit vector \(\hat{u}\) and step size \(\sigma\), and an action space \(A\) as the set of all possible actions:
$$A=\left\{(\hat{u},\sigma):\hat{u}\in{\mathbb{R}} \space \land \|\hat{u}\|=1 \space\land\space 0 \leq \sigma \leq 1\right\}.$$

\section{Methodology}
The VariAntNet pipeline, illustrated in Fig.~\ref{fig: VariAntNet pipeline}, processes local observations sensed by each agent. It begins by applying a rotational transformation preprocessing step, described in \ref{sec: preprocessing}. The rotated observations are then fed into the NN, detailed in \ref{sec: nn architecture}, which generates the control action consisting of a movement direction and step size for the agent.

\begin{figure*}[h!]
\begin{center}
\includegraphics[width=0.9\textwidth]{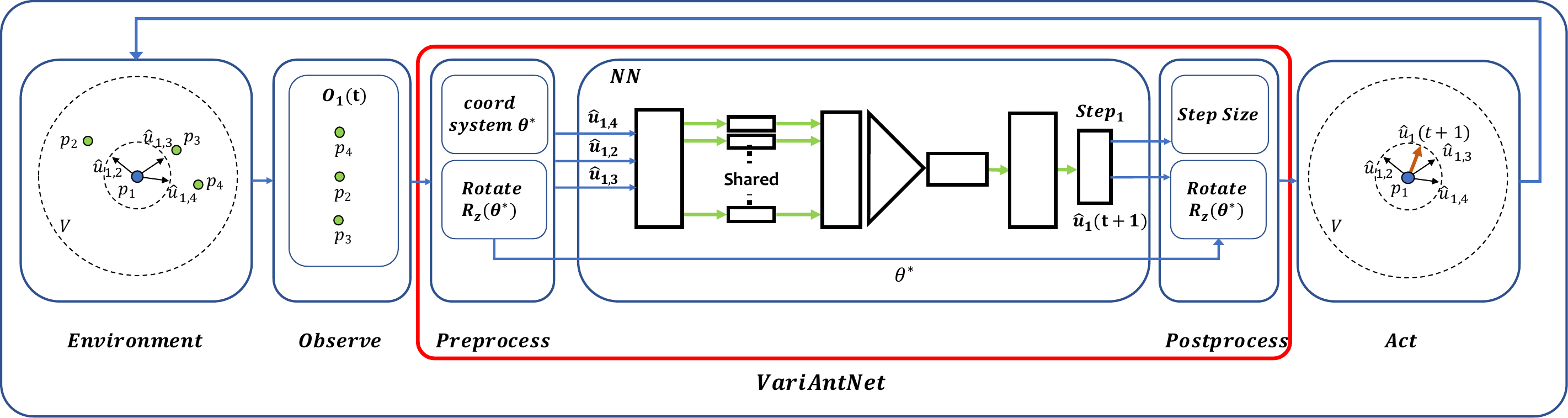}
\caption{The VariAntNet pipeline includes a preprocessing and postprocessing rotational equivariant transformation and an order invariant, variable input size NN. The red arrow in the `act' component represents the direction and the step size of the agents' motion.}
\label{fig: VariAntNet pipeline}
\end{center}
\end{figure*}

\subsection{Preprocessing Stage}\label{sec: preprocessing}
As each agent lacks information on its global position and cannot rely on a global coordinate system, its observations are relative to its own arbitrary coordinate system, represented by its x-axis, denoted as \(\hat{x}\). Moreover, geometrically identical local constellations may be observed differently depending on the agent’s rotation.
VariAntNet includes a preprocessing stage that generates a rotational equivariant representation of the observations, guaranteed through analytical transformations.
\newtheorem{definition}{Definition}
\begin{definition}
A function \(m:O\to A\) is considered equivariant w.r.t a transformation \(B \iff \forall O_i(t)\): \(m(B(O_i(t))) = B(m(O_i(t)))\).
\end{definition}
Here \(B\) is any two-dimensional rotational transformation defined as \(B(O_i(t)) = RO_i(t)\), where \(R\in SO(2)\), the group of all two-dimensional rotation matrices. 

This property, depicted in Fig.~\ref{fig: rotational equivariance}, ensures that when an agent's observation is rotated, the corresponding action undergoes the same rotation. 

To achieve this property, a coordinate system is defined based on the observation. The vector \(\overline{u}\) defines the x-axis of the new coordinate system: 
$$\overline{u}(O_i(t))=\frac{\sum_{j\in {\mathcal{N}}_i(t)}\hat{u}_{ij}(t)}{\|\sum_{j\in {\mathcal{N}}_i(t)}\hat{u}_{ij}(t)\|}=\frac{O_i(t){\mathbf{1}}_n}{\|O_i(t){\mathbf{1}}_n\|},$$
where \({\mathbf{1}}_n\) is a vector of ones of size \(n\).
In case \(\|O_i(t){\mathbf{1}}_n\|=0 \) no rotation is applied.
To rotate the observation to this coordinate system, the angle \(\theta^*\) between \(\overline{u}\) and \(\hat{x}\) is calculated by:
$$\theta^*=\cos^{-1}(\hat{x}^T\overline{u})\cdot\mathrm{sign}(\overline{u}_y)$$

The pre-processing transformation operation is defined as:
$$T_{\text{pre}}(O_i(t))=R(-\theta^*)^TO_i(t)$$
which rotates the observation by \(-\theta^*\) to the new coordinate system defined by $\overline{u}$,
and the post-processing transformation operation is defined as:
$$T_{\text{post}}(a)=R(\theta^*)a$$
which rotates the action of the NN output \(a\in A\) by \(\theta^*\) back to the original coordinate system of the observation.\\
\newtheorem{lemma}{Lemma}

\begin{lemma}
  For any function \(m:O\to A\) and observation \(O_i(t)\in O\) s.t \(\|O_i(t){\mathbf{1}}_n\|\neq 0\) it holds that $T_{\text{post}}(m(T_{\text{pre}}(O_i(t))))$ is \textbf{equivariant} w.r.t $R\in SO(2)$.\\
\end{lemma}

\begin{proof}
    Let \(O_i(t)\in O\), \(R\in SO(2)\) and \(0\leq\theta\leq2\pi\) s.t \(R=R(\theta)\). Let \(\hat{x}\) be the x-axis direction in the coordinate system of \(O_i(t)\). Let \(\theta' =\theta^*(O_i(t))\) be the angle between \(\overline{u}(O_i(t))\) and \(\hat{x}\), and define \(R'=R(\theta')\). \\
    Since \(R\in SO(2)\), for all \(v\in{\mathbb{R}}^2\) it holds that \(\|Rv\|=\|v\|\), therefore:
    \[\overline{u}(RO_i(t))=\frac{RO_i(t)\mathbf{1}_n}{\|RO_i(t)\mathbf{1}_n\|}=R\frac{O_i(t)\mathbf{1}_n}{\|O_i(t)\mathbf{1}_n\|}=R\overline{u}(O_i(t))\]
    As a result the angle between \(\overline{u}(RO_i(t))\) and \(\overline{u}(O_i(t))\) is \(\theta\). Therefore:
    $$\theta^*(RO_i(t))=\theta'+\theta$$
    So it holds that:
    $$T_\text{pre}(RO_i(t))=R(\theta^*(RO_i(t)))^TRO_i(t)=$$
    $$=R(\theta'+\theta)^TRO_i(t) =R'^TR^TRO_i(t)=$$
    $$=R'^TO_i(t)=T_\text{pre}(O_i(t))$$
    Since \(m\) is a function, it follows that \(m\)'s output \(a\in A\) is the same for both \(O_i(t)\) and \(RO_i(t)\):
    $$a=m(T_\text{pre}(RO_i(t)))=m(T_\text{pre}(O_i(t)))$$
    By consequence, it holds that:
    $$T_{\text{post}}(m(T_{\text{pre}}(RO_i(t))=RR'm(T_{\text{pre}}(RO_i(t))=$$
    $$=RT_{\text{post}}(m(T_{\text{pre}}(O_i(t))$$
\end{proof}

\begin{figure}[h!]
    \begin{minipage}[t]{0.45\columnwidth}
        \centering
        \includegraphics[width=3cm]{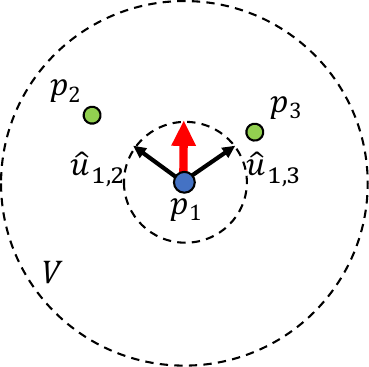}
        \parbox{\columnwidth}
        {\caption*{(a) Bearing-only Sensing of agent \(p_1\).}}
        \label{fig: A}
    \end{minipage}\hfill
    \begin{minipage}[t]{0.45\columnwidth}
        \centering
        \includegraphics[width=3cm]{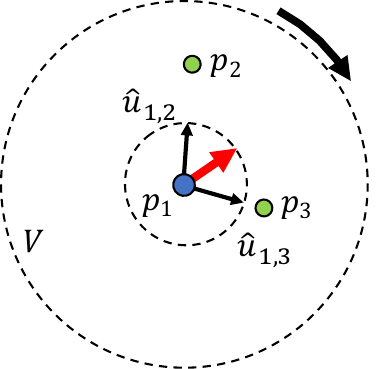}
        \parbox{\columnwidth}
        {\caption*{(b) Bearing-only Sensing of agent \(p_1\) after rotation.}}
        \label{fig: B}
    \end{minipage}
    \vspace{-10pt}
    \caption{Illustrations \((a)\) and \((b)\) demonstrate rotational equivariance. The constellations and input representation are considered equivalent for \(p_1\)'s decision process.}
    \label{fig: rotational equivariance}
    \vspace{-10pt}
\end{figure}

\subsection{VariAntNet's NN Architecture}\label{sec: nn architecture}
VariAntNet's NN architecture, inspired by PointNet ~\cite{qi2017pointnet}, consists of two blocks as illustrated in Fig.~\ref{fig: architecture functions}. The first NN block, described by $f_\phi$, extracts local geometric features from each detected agent. These features are aggregated by a global max pooling operator to a local feature vector \(LF\), which represents the agent's detected neighborhood. The \(LF\) is then 
passed to the following NN block, described by $g_\psi$, which produces the agent's action \(a\in A\). The NN can be mathematically described as:
$$LF_i(t)=\max_{j\in{\mathcal{N}}_{i}(t)}f_{\theta}(u_{ij}(t)), \quad a_i(t)=g_{\psi}(LF_i(t))$$

This architecture offers several key benefits:
\begin{itemize}
    \item \textbf{Variable Input Size:} By max pooling, the network can process inputs of varying sizes, since max pooling does not have size constraints.
    \item \textbf{Invariance to Input Order:} 
    Max pooling ensures that the order of observations does not affect the output of $LF$ representation, which serves as input to $g_{\psi}$.
    \item \textbf{Small Parameter Count:} The network’s architecture ensures that the number of parameters isn't affected by the swarm size. The network maintains a small parameter count, achieving high computational efficiency, as shown in Table~\ref{tab: nn_parameters}. The network controls the local behavior at over 1,000 FPS on a CPU, making it suitable for a deployment on low computational power.
\end{itemize}

\begin{figure}[h]
    \begin{center}
        \includegraphics[width=0.8\columnwidth]{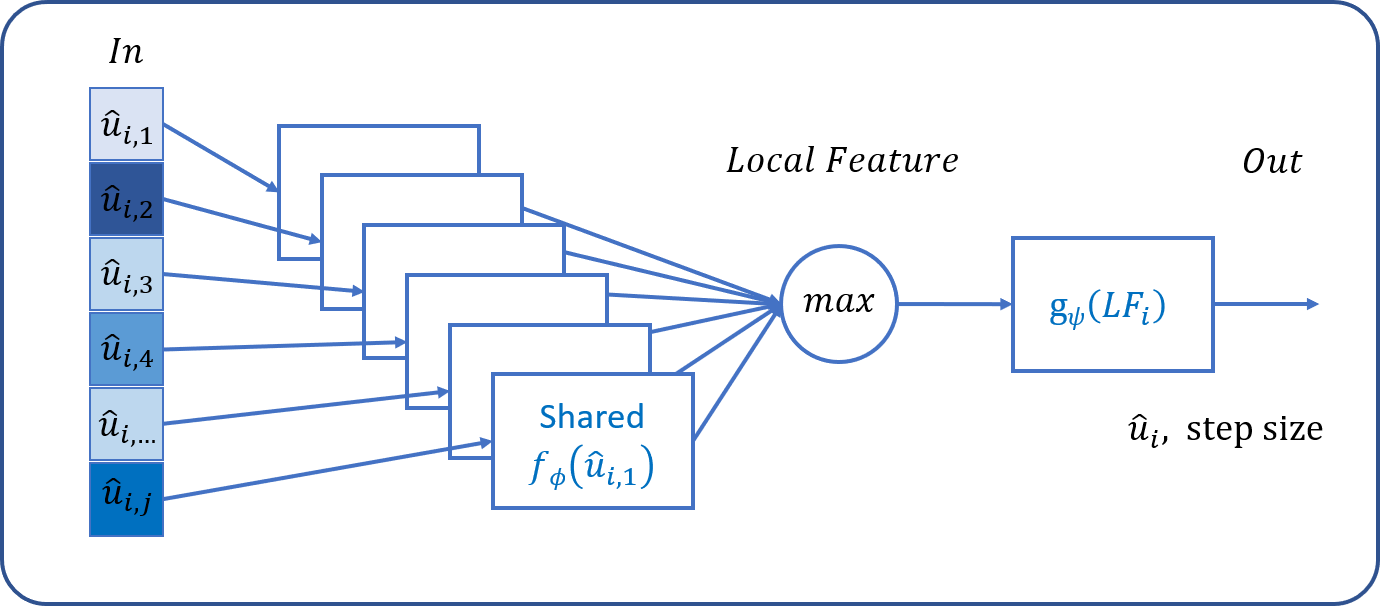}
        \caption{NN structure consists of $f_\phi(\hat u_{i,j})$, a shared MLP that encodes observations, a max pooling layer for extracting the agent's local features, and a final MLP denoted $g_{\psi}(LF(f_\phi(\hat u_{i,j})))$, which determines the agent's next step direction and size.}
        \label{fig: architecture functions}
    \end{center}
\end{figure}

\setlength{\tabcolsep}{4pt}
\begin{table}[h]
\centering
\normalsize
\caption{The VariAntNet pipeline consists of an 8-layer neural network (NN) with 3,875 parameters.}
\renewcommand{\arraystretch}{1.1}
\begin{tabular}{c l c c c c }
    \hline
    \textbf{Layer}  & \textbf{Type}   & \textbf{In}    & \textbf{Out}  &  \textbf{Act. func.}   & \textbf{Params} \\
    \hline
    1               & Shared FC       & \( 2 \)           & \( 16 \)         & \(Tanh \)              & 48        \\ 
    2               & Shared FC       & \( 16 \)          & \( 32 \)         & \(Tanh\)               & 544       \\ 
    4               & Shared FC       & \( 32 \)          & \( 16 \)         & \(Tanh\)               & 528       \\ \hline     
    3               & Max Pool       & \( 16 \)          & \( 16 \)         & -                      & -         \\ \hline
    4               & FC              & \( 16\)           & \( 32 \)         & \(Tanh\)               & 544       \\ 
    5               & FC              & \( 32 \)          & \(32 \)          & \(Tanh \)              & 1,056      \\ \hline    
    7               & Direction (FC)  & \(32 \)           & \( 2 \)          & \(Tanh\)               & 66        \\ 
    7               & FC       & \(32 \)           & \( 32 \)          & \(Sigmoid\)            & 33        \\ 
    8               & Step size (FC)       & \(32 \)           & \( 1 \)          & \(ReLU\)            & 1,056        \\ 
    
    \hline     
    \textbf{Total}  & -               & -                 & -                & -                      & \textbf{3,875} \\
    \hline
\end{tabular}
\label{tab: nn_parameters}
\end{table}

\subsection{Geometric Loss Function} \label{sec: Geometric Loss Function}
VariAntNet's loss function \(\mathcal{L}\) is composed of two functions.
While the swarm's primary task can be quantified using a single task loss function, \(\mathcal{L_{\text{Task}}}\), it tends to disconnect in order to reach local minima of \(\mathcal{L_{\text{Task}}}\). To address this, an additional loss function, \(\mathcal{L}_{\text{Cohesiveness}}\), is introduced to measure swarm connectivity. This secondary loss function is required to be continuous, differentiable, and computationally efficient, ensuring effective gradient propagation without impeding the training process.
The loss, \(\mathcal{L}\), is defined as:
$${\mathcal{L}}= \alpha\mathcal{L}_{\text{Cohesiveness}} + \beta\mathcal{L}_{\text{Task}}$$
The constants $\alpha$ and $\beta$ enable amplification of each one of the loss functions, thus affecting \(\mathcal{L}\). The values that were used are $\alpha=1$ and $\beta =1 $ Wherever it is not mentioned \\

\subsubsection{Task Loss}
The task loss, denoted by $\mathcal{L}_\text{Task}$, quantifies the swarm's success in fulfilling the task requirements.

To enhance learning efficiency, the maximum distance from the swarm's center of mass is used, defined as:
$${\mathcal{L}}_{Task}(t)=\max_{p_i\in P(t)} \left\|p_i(t)-\frac{1}{|P(t)|}\sum_{p_j\in P(t)}p_j(t)\right\|$$


\subsubsection{Cohesiveness Loss}
The Cohesiveness loss function is based on a continuous and differentiable bound for Cheeger's constant~\cite{cheeger1970lower} of the swarm's visibility graph.
The visibility graph is enhanced by introducing an edge-weighting function that reflects the proximity between neighboring agents. Each edge weight reaches its maximum when the connected agents are located at the same position and decreases smoothly as the distance between them increases, eventually reaching zero when mutual visibility is lost. The distances are not available to the agents and are used only for the NN's training. 
The weight function is defined as:
$$ w: {\E} \to {\mathbb{R}}, \quad w(\nu_i, \nu_j)=V - d(p_i,p_j).$$
Similarly to \cite{walchesseneigenvalues}, this definition is extended for sets of vertices, defined as the sum of weighted degrees of the vertices in the set:
$$w:{\mathcal{P}}({\V}) \to {\mathbb{R}},\quad w(\Omega)=\sum_{\nu_i\in\Omega}\sum_{\{\nu_i,\nu_j\}\in {\E}}w(\nu_i,\nu_j)$$
and also for sets of edges, defined as the sum of weights of all edges in the set:
$$w:{\mathcal{P}}({\E}) \to {\mathbb{R}},\quad w(S)=\sum_{\{\nu_i,\nu_j\}\in S}w(\nu_i,\nu_j).$$
The edge boundary of $\Omega$ is defined as the set of edges connecting sub-graph \(\Omega\) to the rest of the graph:
$$\partial\Omega=\big\{\{\nu_i,\nu_j\}\in {\E}:\space \nu_i\in \Omega \space \land \space \nu_j \in {\V}/\Omega\big\}.$$
The Cheeger's constant of the visibility graph, shown in \cite{cheeger1970lower} quantifies its connectivity \cite{marsden2013eigenvalues}\cite{anderson1985eigenvalues}, and is defined as:
$$h(G)=\underset{\Omega\subset{\V},\space w(\Omega)\leq\frac{1}{2}w({\V})}{\min} \left\{ \frac{w(\partial\Omega)}{w(\Omega)}\right\}$$
where \(w(\partial \Omega)\) quantifies the connectivity of \(\Omega\) to the rest of the graph, and \(w(\Omega)\) quantifies the inner connectivity of \(\Omega\).
Approaching disconnection leads to a decrease in \(h(G)\), as the swarm tends to separate into subgroups rather than into individually disconnected agents. Hence, when the swarm approaches separation, the inner connectivity of each component, \(w(\Omega)\), is unaffected while the global connectivity between the components, \(w(\partial\Omega)\), decreases.
Therefore, by minimizing \(\frac{1}{h(G)}\), the NN learns to keep the swarm cohesive. 

Furthermore, when the swarm approaches disconnection, this function becomes infinitely bigger than the \(\mathcal{L}_\text{Task}\), focusing the learning on cohesiveness. 
Unfortunately, calculating \(h(G)\) directly is impractical, making it necessary to rely on a lower bound.\\

As shown in \cite{walchesseneigenvalues}, such a bound can be achieved by utilizing properties of the Laplacian matrix.
The Laplacian matrix of \(G\) is defined by:
$$L=D-A,$$
where \( A \) is the weighted Adjacency matrix: 
    $$A_{ij}=\begin{cases}w(\nu_i,\nu_j) & \{\nu_i,\nu_j\}\in \E\\0 & \text{otherwise}\end{cases}$$
and \( D \) is the Diagonal degree matrix $$D_{ij} = \begin{cases}\sum_{\{\nu_i,\nu_k\}\in \E} w(\nu_i,\nu_k) & i=j \\ 0 & \text{otherwise}\end{cases}$$
By analyzing its eigenvalues, denoted as:
$$0=\lambda_1\leq \lambda_2\leq\dots \lambda_n,$$ and utilizing \(\lambda_2\), also known as the algebraic connectivity, a lower bound of \(h(G)\) is achieved:
$$h(G)\geq\frac{\lambda_2}{2}$$
By using this bound, the loss function \(\frac{1}{h(G)}\) can be optimized using \(\frac{1}{\lambda_2}\), since \(\lambda_2\) is continuous and differentiable with respect to the graph's weights, and can be efficiently computed during training. The cohesiveness loss function is defined as:
$${\mathcal{L}}_\text{Cohesiveness}=\frac{1}{\lambda_2(G)+\varepsilon}$$
where \(0<\varepsilon \ll 1\) is a small value that is added for stability. In our case, the value of \(\varepsilon\) is chosen to be \(10^{-6}\).

\section{Datasets and Curriculum}
The agent's initial constellation significantly influences the swarm's ability to maintain cohesiveness, as agents may lose connection with their neighbors positioned at the edge of their visibility range. 
In our work, a dataset generator was developed to produce initial constellations under varying geometric constraints, used for training, validation, and evaluation.

\subsection{Dataset Features}
The dataset constellation generator produces randomly positioned, connected initial constellations with varying numbers of agents, visibility ranges, and difficulty levels.
The difficulty level is determined by the initial spatial distribution of the agents. We define the effective visibility range, $V_{eff}$= $V\times VR$, as the product of the visibility range $V$ and a visibility ratio $VR$, where $0< VR \leq 1$. During constellation generation, agents are placed such that each agents has at least one neighbor within its $V_{eff}$. As $VR$ increases, so does the difficulty level, since the potential for disconnection rises accordingly. Our predefined difficulty levels are described in Fig.~\ref{fig: constellation generator fig}.

\begin{figure}[h!]
    \begin{minipage}[t]{0.29\columnwidth}
        \centering
        \includegraphics[width=1\columnwidth]{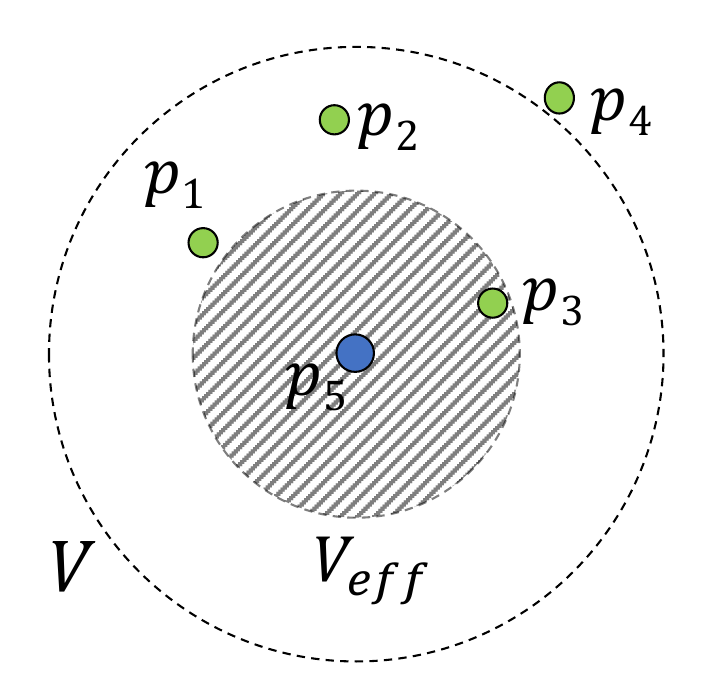}
        \caption*{\centering (a) Regular constellation $V_{eff}=0.5 \times V$}
        \label{fig: easy costellation}
    \end{minipage} \hspace{0.29cm}
    \begin{minipage}[t]{0.30\columnwidth}
        \centering
        \includegraphics[width=1\columnwidth]{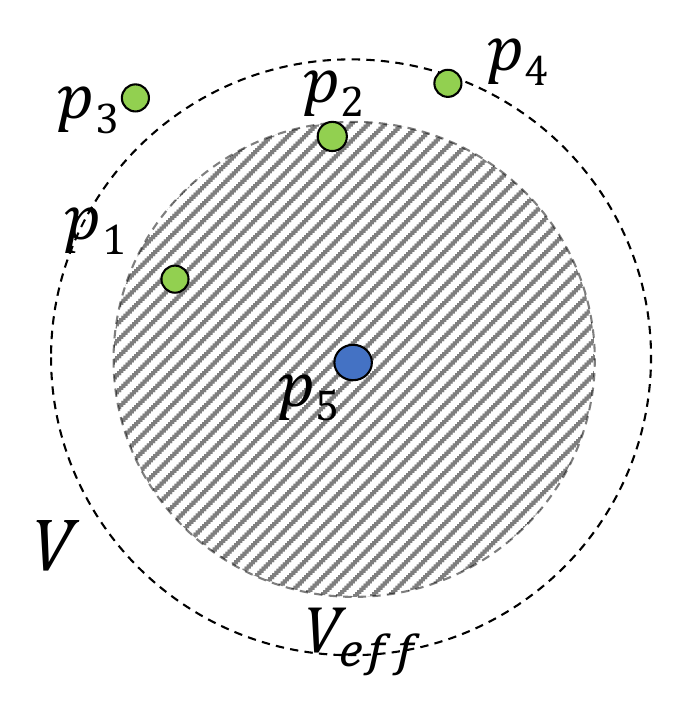}
        \caption*{\centering (b) Challenging constellation $V_{eff}=0.75 \times V$}
        \label{fig: medium constellation}
    \end{minipage}  \hspace{0.29cm}
    \begin{minipage}[t]{0.30\columnwidth}
        \centering
        \includegraphics[width=1\columnwidth]{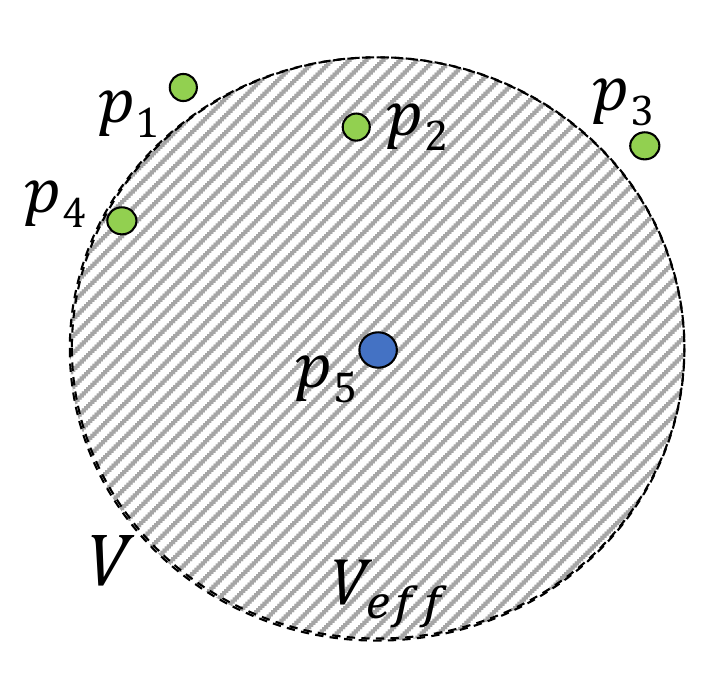}
        \caption*{\centering (c) Marginal constellation $V_{eff}=1 \times V$}
        \label{fig: Marginal constellation}
    \end{minipage}
    \vspace{-5pt}
    \caption{
    The illustrations depict a successful positioning for \(p_5\) in various initial constellations, since in every gray area there is at least one neighboring agent.}
    \label{fig: constellation generator fig}
    \vspace{-10pt}
\end{figure}

\subsection{Constellation Generator}
The initial constellation generator is described by a pseudocode of Algorithm \ref{alg: constellation}.
The flow begins by setting the effective visibility range, $V_{eff}$, determined by $VR$. The first agent is randomly placed within a predefined bounded area. To achieve uniformity, agents are added sequentially, each randomly positioned. If the newly placed agent lies within a distance of $V_{eff}$ of any previously placed agent, the process continues. Otherwise, a new position is sampled until this criterion is met.



To prevent any information from being embedded in the configuration (e.g., ensuring that $p_1$ will not always sense the agent with the subsequent index $p_2$), the indices of the agents are shuffled.


\begin{algorithm}[h!]
\caption{Constellation Generator}\label{alg: constellation}
\begin{algorithmic}[1]
\Require $num\_agents, visibility, boundary, seed, \newline visibility\_ratio$
\Ensure $connected$ $constellation$ or $raise$ $error$
\State $V_{eff} \gets visibility \times visibility\_ratio$
\State $P \gets \{\}$
\State $set\_random\_seed(seed$)
\For{$tries$ in $range(max\_attempts)$}
    \State Generate initial position: $P \gets rand(boundary)$
    \For{$index$ in $num\_robots - 1$}
        \For{$attempts$ in $max\_attempts$}
            \State $p_{next} \gets rand(boundary)$
            \State $distances \gets norm(P,p_{next})$
            \State \text{\textbf{if } $ any(distances < V_{eff})$ \textbf{then}} 
            \State \textbf{\quad break}
            
        \EndFor
            \State \text{\textbf{if }$attempts = max\_attempts$ \textbf{then}}
            
            \State \text{\quad \textbf{break}} \Comment{fail to position agent}
            \State \textbf{else:}
            \State $\quad boundary \gets boundary \cup \text{visibility}(p_{next})$
            \State $\quad P \gets P \cup p_{next}$  
    \EndFor
    \State \text{\textbf{if }$index < num\_agents$ \textbf{then}} 
    \State \textbf{\quad break} \Comment{reinit constellation}  \par
\EndFor
\State \text{\textbf{if }$tries = max\_attempts$ \textbf{then}} \par
\State \quad \textbf{raise error:} 'Failed to create constellation'

\State $P \gets shuffle(P)$
\State \Return $P$
\end{algorithmic}
\end{algorithm}

\subsection{Curriculum}

VariAntNet variations were trained using curriculum learning to ensure training stability, as summarized in Table~\ref {tab: curriculum}. Without this strategy, exposing the network to scenarios with high \(V_{eff}\) early in training frequently leads to swarm fragmentation, resulting in elevated loss values and driving the model into a negative feedback loop.

A dataset is produced for each stage in the curriculum using Algorithm~\ref{alg: constellation} with varying $VR$. The number of agents used for the training datasets is 10. The models are trained on a dataset with various numbers of epochs, steps per scenario, and a changing number of environments (scenarios). The stages gradually increase the level of difficulty and number of steps to train the VariAntNet in initial and final constellations.
\setlength{\tabcolsep}{4pt}
\begin{table}[h]
\centering
\caption{VariAntNet training stages, divided into 3 groups: A, B and C, according to the number of epochs, steps, environments, and learning rate.}
\renewcommand{\arraystretch}{1.1}
\begin{tabular}{c c c c c c c }
    \hline
    \multicolumn{2}{c}{\textbf{Stage}}  & $\mathbf{VR}$  & \textbf{Epochs}    & \textbf{Steps}  &  \textbf{Env.}   & \textbf{Learning} \\
    \textbf{}       & \textbf{ }      & \textbf{}     & \textbf{}    &  \textbf{}     & \textbf{} & \textbf{Rate}\\ \hline
    \multirow{4}{*}{\rotatebox{0}{A}} & 1 & 0.3 & 5 & 200 & 120 & \(5 \mathrm{e}{-5}\) \\
    & 2 & 0.4 & 5 & 200 & 120 & \(5 \mathrm{e}{-5}\) \\
    & 3 & 0.5 & 10 & 200 & 120 & \(5 \mathrm{e}{-5}\) \\
    & 4 & 0.6 & 25 & 200 & 120 & \(5 \mathrm{e}{-5}\) \\
    \hline
    \multirow{4}{*}{\rotatebox{0}{B}} & 5 & 0.65 & 10 & 200 & 120 & \(5 \mathrm{e}{-6}\) \\
    & 6 & 0.7 & 35 & 200 & 120 & \(5 \mathrm{e}{-6}\) \\
    & 7 & 0.75 & 35 & 200 & 120 & \(5 \mathrm{e}{-6}\) \\
    & 8 & 0.8 & 35 & 200 & 120 & \(5 \mathrm{e}{-6}\) \\
    \hline
    \multirow{6}{*}{\rotatebox{0}{C}} & 9 & 0.75 & 35 & 500 & 180 & \(5 \mathrm{e}{-6}\) \\
    & 10 & 0.8 & 35 & 500 & 180 & \(5 \mathrm{e}{-6}\) \\
    & 11 & 0.85 & 35 & 500 & 180 &\(5 \mathrm{e}{-6}\) \\
    & 12 & 0.9 & 35 & 500 & 180 & \(5 \mathrm{e}{-6}\)\\
    & 13 & 0.95 & 35 & 500 & 180 & \(5 \mathrm{e}{-6}\) \\
    & 14 & 1.0 & 35 & 500 & 180 & \(5 \mathrm{e}{-6}\) \\
    \bottomrule
    \hline
\end{tabular}
\label{tab: curriculum}
\end{table}


\section{Experiments and Results}\label{sec: experiments and results}

\begin{figure*}[!ht]
    \begin{minipage}[t]{0.32\textwidth}
        \centering
        \includegraphics[width=1\textwidth]{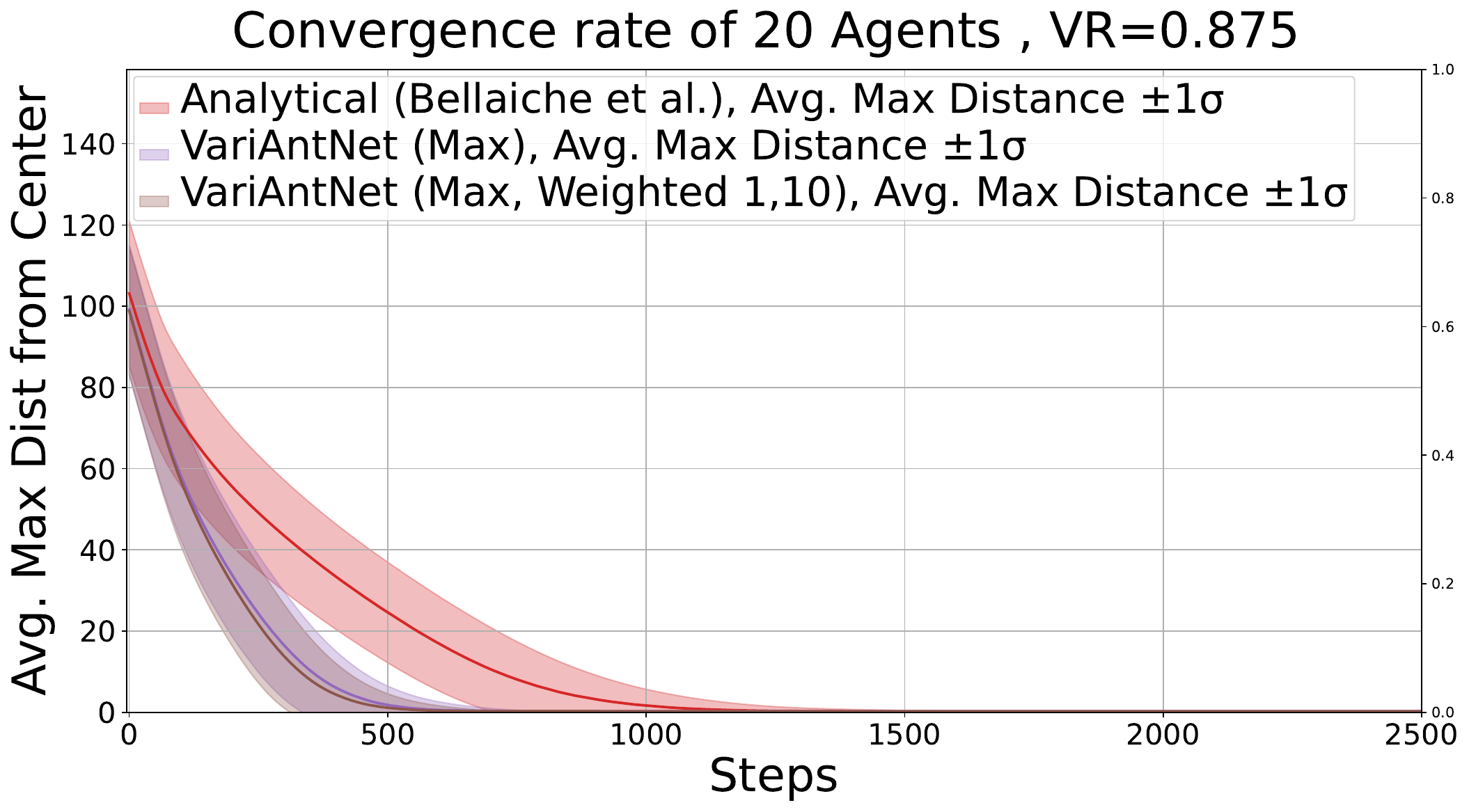}
        \parbox{0.9\textwidth}{\caption*{(a) 20 Agents, initial $VR=0.875$}}
        \label{fig: Regular level costellation}
    \end{minipage}\hfill
    \begin{minipage}[t]{0.32\textwidth}
        \centering
        \includegraphics[width=1\textwidth]{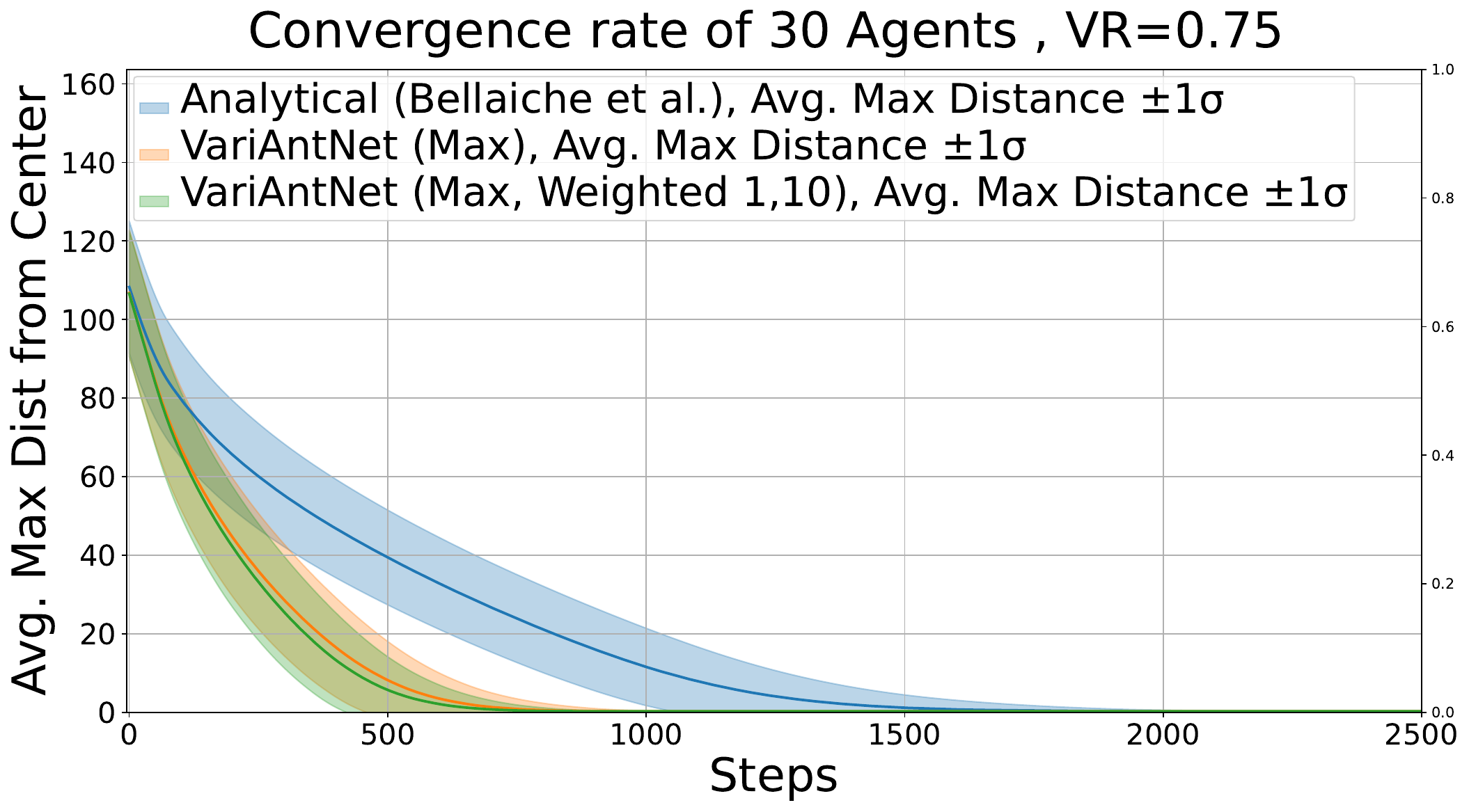}
        \parbox{0.9\textwidth}{\caption*{(b) 30 Agents, initial $VR=0.75$}}
        \label{fig: Challenging level constellation}
    \end{minipage}
    \begin{minipage}[t]{0.32\textwidth}
        \centering
        \includegraphics[width=1\textwidth]{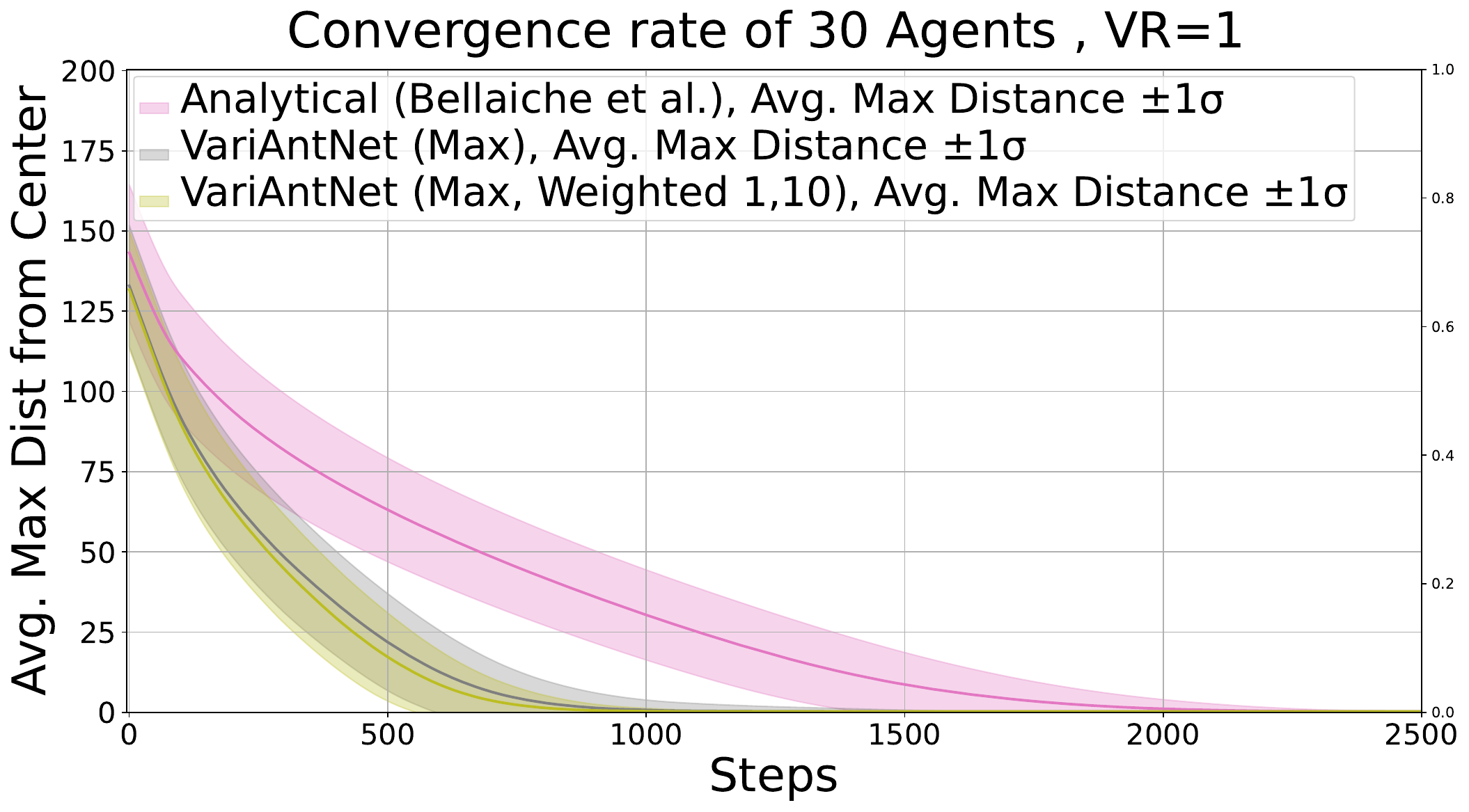}
        \parbox{0.9\textwidth}{\caption*{(c) 30 Agents, initial $VR=1$}}
        \label{fig: Marginal level constellation}
    \end{minipage}
    \vspace{-15pt}
    \caption{Comparison of the average convergence steps of the swarm between the VariAntNet variants and the Analitical Algorithm (Bellaiche~\cite{bellaiche2017continuous}) with varying difficulty levels and swarm sizes. In each plot, the line and the shaded area represent the mean of the maximum distance and its distribution from the swarms' center of mass over 1,000 environments.
    }
    \label{fig: convergence results}
    \vspace{4pt}
\end{figure*}

\begin{table*}[!ht]
\caption{
Comparison of analytical Bellaiche~\cite{bellaiche2017continuous} and VariAntNet models across five difficulty levels of initial swarm configuration: \textquoteleft{}Regular\textquoteright{}, \textquoteleft{}Challenging\textquoteright{}, \textquoteleft {}Marginal\textquoteright{}, VR=0.625, and VR=0.875. Each configuration is evaluated on 1,000 scenarios of 10, 20, and 30 agents. Each cell contains three metrics: (i) Steps, the average convergence number of steps across successfully converged scenarios, (ii) Conn.\%, the percentage of scenarios in which the swarm remained connected throughout the simulation, and (iii) Conv.\%, the percentage of scenarios in which the swarm succeeded in converging. Note that in some scenarios, the analytical Bellaiche~\cite{bellaiche2017continuous} algorithm did not converge within the desired time limit.
}
\centering
\renewcommand{\arraystretch}{1.1}
\resizebox{0.9\textwidth}{!}{
\begin{tabular}{l l | r c c | r c c | r c c}
\hline
\textbf{Level} & \textbf{Model} & \multicolumn{3}{c|}{\textbf{10 Agents}}& \multicolumn{3}{c|}{\textbf{20 Agents}}& \multicolumn{3}{c}{\textbf{30 Agents}} \\
(VR)& & Steps & Conn.\% & Conv.\% & Steps & Conn.\% & Conv.\% & Steps & Conn.\% & Conv.\% \\
\hline
Regular & Analytical (Bellaiche et al.) & 177 & 100.0 & 100.0 & 468 & 100.0 & 100.0 & 828 & 100.0 & 100.0 \\
 & VariAntNet (Max) & 106 & 100.0 & 100.0 & 228 & 100.0 & 100.0 & 370 & 100.0 & 100.0 \\
 & VariAntNet (Max, Weighted 1,10) & 102 & 100.0 & 100.0 & 211 & 100.0 & 100.0 & 337 & 99.9 & 99.9 \\
\hline
0.625 & Analytical (Bellaiche et al.) & 225 & 100.0 & 100.0 & 593 & 100.0 & 100.0 & 1038 & 100.0 & 100.0 \\
 & VariAntNet (Max) & 136 & 100.0 & 100.0 & 293 & 99.2 & 99.2 & 467 & 98.1 & 98.1 \\
 & VariAntNet (Max, Weighted 1,10) & 131 & 100.0 & 100.0 & 271 & 98.7 & 98.7 & 427 & 97.8 & 97.8 \\
\hline
Challenging & Analytical (Bellaiche et al.) & 271 & 100.0 & 100.0 & 702 & 100.0 & 100.0 & 1235 & 100.0 & 100.0 \\
 & VariAntNet (Max) & 165 & 99.5 & 99.5 & 346 & 95.2 & 95.2 & 554 & 91.6 & 91.6 \\
 & VariAntNet (Max, Weighted 1,10) & 159 & 98.8 & 98.8 & 322 & 94.9 & 94.9 & 506 & 90.6 & 90.6 \\
\hline
0.875 & Analytical (Bellaiche et al.) & 314 & 100.0 & 100.0 & 837 & 100.0 & 100.0 & 1450 & 100.0 & 99.8 \\
 & VariAntNet (Max) & 190 & 95.2 & 95.2 & 408 & 81.1 & 81.1 & 638 & 69.6 & 69.6 \\
 & VariAntNet (Max, Weighted 1,10) & 182 & 93.6 & 93.6 & 379 & 78.7 & 78.7 & 587 & 66.8 & 66.8 \\
\hline
Marginal & Analytical (Bellaiche et al.) & 362 & 100.0 & 100.0 & 950 & 100.0 & 100.0 & 1629 & 100.0 & 99.0 \\
 & VariAntNet (Max) & 217 & 84.5 & 84.5 & 451 & 58.4 & 58.4 & 711 & 45.8 & 45.8 \\
 & VariAntNet (Max, Weighted 1,10) & 209 & 82.6 & 82.6 & 418 & 53.8 & 53.8 & 646 & 41.3 & 41.3 \\
\hline
\end{tabular}
}
\label{tab:model_comparison}
\end{table*}

\begin{figure*}[!ht]
\centering
\includegraphics[width=\textwidth]{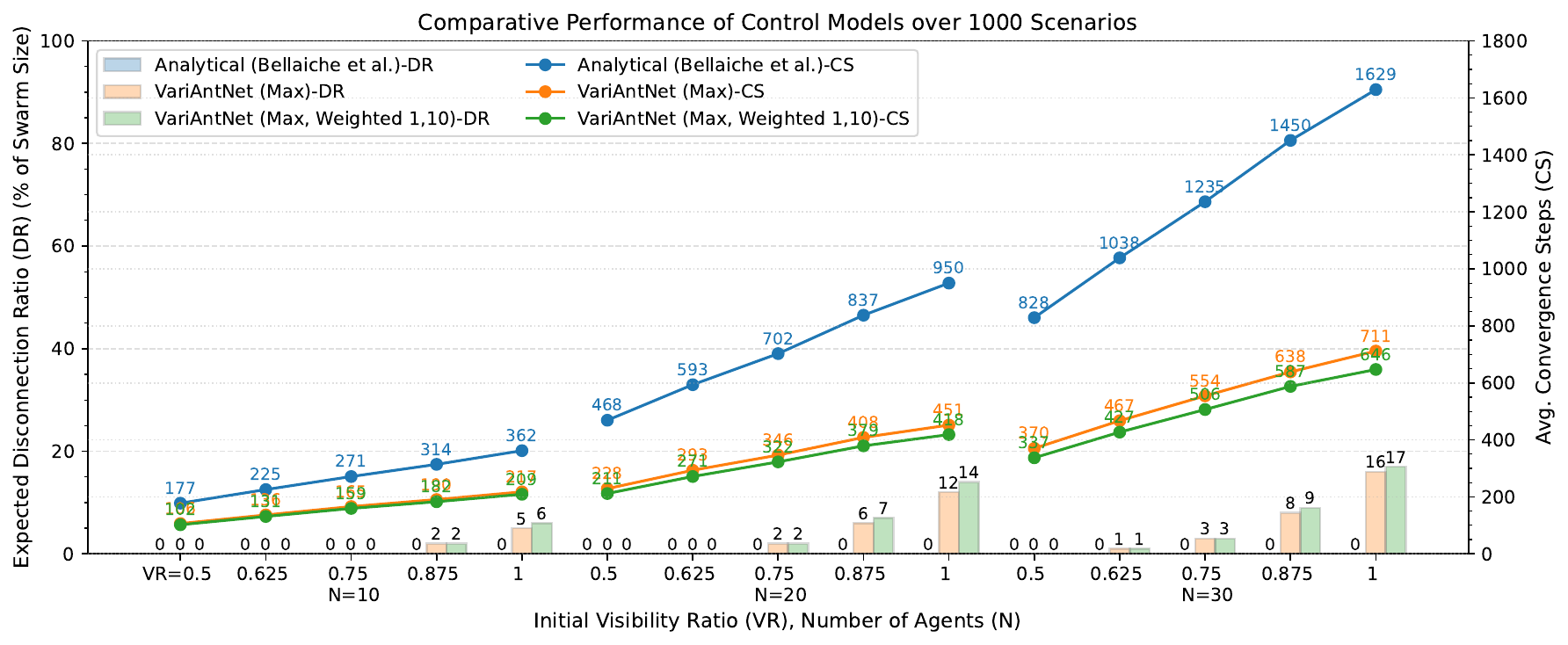}
\vspace{-10pt}
\caption{Comparison of control models based on Expected Disconnection Ratio (DR, bars) and Average Convergence Steps (CS, lines) across 1,000 scenarios. DR is defined as the expected disconnected ratio, the percentage of agents that separate from the largest group. The DR is calculated as a percentage of the initial swarm size. Each group corresponds to a unique combination of the initial constellation's $VR$ and the number of agents $N$.
VariAntNet variants outperform the analytical baseline in convergence speed, particularly under challenging conditions, yet have some degradation in 30-Agent scenarios where the initial configuration exceeds $VR=0.875$.
}
\label{fig: Model Comparison}
\end{figure*}

\subsection{Evaluation Dataset}\label{sec: evaluation dataset}
To evaluate VariAntNet, five types of difficulty-level datasets were generated, each with 1,000 randomized environments. The levels are defined by the initial visibility ratio (AR): \textquoteleft Regular\textquoteright\ with $VR=0.5$, \textquoteleft Challenging\textquoteright\ with $VR=0.75$, and \textquoteleft Marginal\textquoteright\ with $VR=1$ and two  intermediate  levels with $VR=0.625$ and $VR=0.875$. At the \textquoteleft Marginal\textquoteright\ level, the neighboring agents are allowed to be positioned at the edge of the visibility range, significantly increasing the likelihood of a disconnection as depicted in Fig.~\ref{fig: constellation generator fig}(c).

\subsection{Results} \label{sec: Results}
A comparative evaluation was preformed between two VariAntNet variations (see section \ref{sec: Ablation Study}) and the analytical baseline proposed by Bellaiche et al. ~\cite{bellaiche2017continuous}, with a focus on swarm convergence. 
The comparison examines key aspects of convergence efficiency and swarms' connectivity preservation throughout the process. The results are summarized in Table \ref{tab:model_comparison} and visualized in Figs. \ref{fig: convergence results} and \ref{fig: Model Comparison}.
The following VariAntNet variants demonstrated the best performance:
\begin{itemize}
    \item VariAntNet (Max): This variation employs a max pooling aggregation function. As shown in Fig. \ref{fig: Model Comparison}, it outperforms in maintaining swarm connectivity, with a slightly slower convergence rate. 
    \item VariAntNet (Max, Weighted 1,10): This variant incorporates a modification of weight $\alpha, \beta$ in the loss function, ${\mathcal{L}}= \alpha\mathcal{L}_{\text{Cohesiveness}} + \beta\mathcal{L}_{\text{Task}}$, as defined in Section \ref{sec: Geometric Loss Function}. By setting $\alpha=1$ and $\beta=10$, the training process places greater emphasis on the gathering task.\\
\end{itemize}

The evaluation of the models was carried out using 1,000 identical initial constellations, with each scenario executed for up to 2,500 steps, a practical upper bound. The performance metrics include the convergence rate, the mean convergence steps (for successfully converged scenarios), the percentage of scenarios in which the swarm remained fully connected, and the partial size of the largest subgroup in case of swarm separation.
\begin{figure*}[t]
\centering
\includegraphics[width=\textwidth]{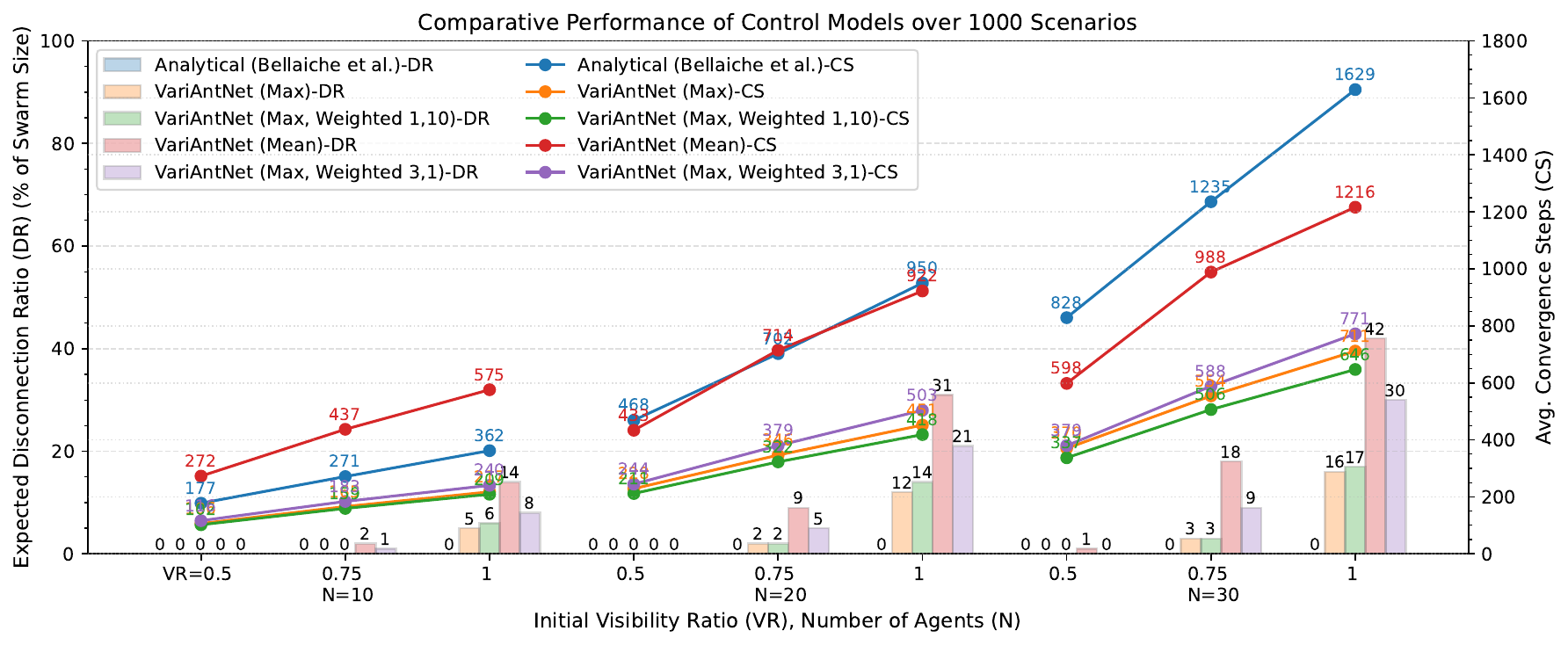}
\vspace{-10pt}
\caption{
Comparison of VariAntNet variants based on the Expected Disconnection Ratio (DR, bars) and Average Convergence Steps (CS, lines) across 1,000 randomized scenarios. DR is defined as the expected size of the largest connected subgroup, expressed as a percentage of the initial swarm size. Each group corresponds to a unique combination of initial constellation parameters: VR and number of agents. VariAntNet (Max) and VariAntNet (Max, Weighted 1,10) outperform the other variants, demonstrating high convergence rates and low expected disconnection across all agent counts and VR values. These variants are therefore selected for deeper analysis in the results section \ref{sec: Results}.
}
\label{fig: Abletion Model Comparison}
\end{figure*}

Our results reveal a trade-off between connectivity and convergence rate.
The VariAntNet (Max, Weighted 1:10) variant significantly outperforms the analytical algorithm proposed by Bellaiche et al. in terms of convergence efficiency. Across multiple settings, it converges up to 2.5 faster than the analytical method. The analytical algorithm converges significantly slower, and in some scenarios, fails to converge in a practical step limit. 
The step limit is defined as 2,500 steps, which is more than 3.5 times the average convergence steps required for VariAntNet (Max, Weighted 1:10).
 
 However, this improved speed comes at the cost of swarm fragmentation. For example, under `Marginal' initial constellation, $(VR = 1)$, the expected disconnection ratio (DR) reaches 17\%, whereas the analytical method maintains connectivity throughout all scenarios. 

In comparison, VariAntNet (Max) converges approximately 7\% slower than the weighted variant but shows slightly improved connectivity. For example, in 20-agent scenarios with an initial constellation of $VR=0.875$ and $VR=1$, the DR is reduced from 7\% to 6\% and from 14\% to 12\%, respectively. This points out that a gain in robustness comes at the expense of slower coordination.

A key challenge in VariAntNet-based models arises from their statistical nature and the agents' inability to sense their sensing boundary. The decentralized policies may lead agents to operate near the edges of their visibility range, increasing the risk of disconnection, particularly in sparse or marginal configurations.

Despite this, the practical acceptability of agent loss is a critical factor. In realistic deployments of low-cost, ant-inspired swarm robots, tolerating our “golden point” up to 10\% disconnection may be acceptable, especially in high-risk or time-sensitive missions. Under such constraints, VariAntNet (Max, Weighted 1:10) offers a compelling advantage, particularly when the initial constellation is under the condition of $VR\leq0.875$, balancing speed with an acceptable level of cohesion loss.


\subsection{Ablation Study} \label{sec: Ablation Study}

To evaluate the contribution of VariAntNets' architectural and algorithmic components, we conducted an ablation study. This study, illustrated in Fig. \ref{fig: Abletion Model Comparison}, assesses the effect of various designs and training  on the convergence rate and cohesion performance: 
\begin{itemize}
    \item \text{Weighted loss function}: By setting various values of $\alpha$ and $\beta$, the cohesiveness loss function $L_{\text{Cohesiveness}}$, effectively trains with different focuses between task and cohesiveness parts. 
    
    \item \text{Aggregation function}: Two variants of the aggregation function were tested: Mean and Max. The Max aggregation function was found to be more effective, as it better captures and represents boundary values.
\end{itemize}

\section{Conclusions} 
\label{sec:conclusion}
In this study, we introduced VariAntNet, a novel NN-based approach for decentralized control of swarms. We evaluated its performance on the multi-agent bearing-only gathering problem and demonstrated that the network successfully learned to perform the task while aiming to maintain swarm cohesion. Compared to an existing analytical method, VariAntNet significantly improved convergence rate across various scenarios, with only a modest reduction in connectivity rate, an acceptable trade-off in many practical settings.
Our findings highlight the strong potential of NN in addressing complex swarm coordination tasks. Moreover, the VariAntNet pipeline is efficient and lightweight for practical use. It is modular by design, allowing its components to be adapted to a wide range of additional swarm-related problems in future research.




\end{document}